\documentclass{article}
\usepackage[preprint,nonatbib]{neurips_2020}
\usepackage[utf8]{inputenc} % allow utf-8 input
\usepackage[T1]{fontenc}    % use 8-bit T1 fonts
\usepackage{hyperref}       % hyperlinks
\usepackage{url}            % simple URL typesetting
\usepackage{booktabs}       % professional-quality tables
\usepackage{amsfonts}       % blackboard math symbols
\usepackage{nicefrac}       % compact symbols for 1/2, etc.
\usepackage{microtype}      % microtypography

\usepackage{algorithm}
\usepackage{algcompatible}
\usepackage[noend]{algpseudocode}

\usepackage{float}
\usepackage{times}
\usepackage{hyperref}
\usepackage{mathtools}
\usepackage{multirow,color,graphicx}
\usepackage{subcaption}
\usepackage{caption}
\usepackage{xfrac,amsfonts,amsbsy}
\usepackage[titletoc]{appendix}
\usepackage{xspace}
\usepackage{savesym,verbatim}
\usepackage[titletoc]{appendix}
\usepackage{paralist}
\usepackage{subscript}
\usepackage{tikz}
\usepackage{verbatim}
\usepackage{amsthm}

\usepackage{bbm}

\newtheorem{lemma}{Lemma}
\newtheorem{theorem}{Theorem}

\newtheorem{definition}{Definition}

\makeatletter

\newcommand{\Rmnum}[1]{\expandafter\@slowromancap\romannumeral #1@}
\makeatother

\DeclareMathOperator*{\argmin}{arg\,min}
\DeclareMathOperator*{\argmax}{arg\,max}

\newcommand{\expct}[1]{\mathbb{E}\left[#1\right]}
\newcommand{\expctover}[2]{\mathbb{E}_{#1}\!\left[#2\right]}
\newcommand{\abs}[1]{\left\vert#1\right\vert}
\def \argmax {\mathop{\rm arg\,max}}
\def \argmin {\mathop{\rm arg\,min}}

\newcommand{\one}{\mathbf{1}}

\newcommand{\bs}{{\mathbf{s}}}

\newcommand{\norm}[1]{\left\lVert#1\right\rVert}
\newcommand{\bcc}[1]{\left\{{#1}\right\}}
\newcommand{\brr}[1]{\left({#1}\right)}
\newcommand{\bss}[1]{\left[{#1}\right]}

\allowdisplaybreaks

\title{Environment Shaping in Reinforcement Learning\\using State Abstraction}

\author{
    Parameswaran Kamalaruban\\
    LIONS, EPFL\\
	\And
    Rati Devidze\\
    MPI-SWS\\
	\And
    Volkan Cevher\\
    LIONS, EPFL\\
    \And
     Adish Singla\\
     MPI-SWS\\
}
\begin{document}
\maketitle

%%%%%%%%%%%%%%%%%%%%%%%%%%%%%%%%%%%%%%%%%%%%%%%%%%%%%%%%%% ABSTRACT
%!TEX root = main.tex
%%%%%%%%%%%%%%%%%%%%%%%%%%%%%%%%%%%%%%%%%%%%%%%%%%%%%%%%%
%%%%%%%%%%%%%%%%%%%%%%%%%%%%%%%%%%%%%%%%%%%%%%%%%%%%%%%%%
\begin{abstract}
\looseness-1
One of the central challenges faced by a reinforcement learning (RL) agent is to effectively learn a (near-)optimal policy in environments with large state spaces having sparse and noisy feedback signals. In real-world applications, an expert with additional domain knowledge can help in speeding up the learning process via \emph{shaping the environment}, i.e., making the environment more learner-friendly.  A popular paradigm in literature is \emph{potential-based reward shaping}, where the environment's reward function is augmented with additional local rewards using a potential function. However, the applicability of potential-based reward shaping is limited in settings where (i) the state space is very large, and it is challenging to compute an appropriate potential function, (ii) the feedback signals are noisy, and even with shaped rewards the agent could be trapped in local optima, and (iii) changing the rewards alone is not sufficient, and effective shaping requires changing the dynamics. We address these limitations of potential-based shaping methods and propose a novel framework of \emph{environment shaping using state abstraction}. Our key idea is to compress the environment's large state space with noisy signals to an abstracted space, and to use this abstraction in creating smoother and more effective feedback signals for the agent. We study the theoretical underpinnings of our abstraction-based environment shaping, and show that the agent's policy learnt in the shaped environment preserves near-optimal behavior in the original environment.
\end{abstract}
\section{Introduction}\label{sec:intro}

Recently, Reinforcement Learning (RL) algorithms~\cite{sutton2000policy,sutton2018reinforcement} have demonstrated tremendous success in complex games~\cite{mnih2015human,silver2017mastering} and simulation environments~\cite{lillicrap2015continuous}. However, in general, their deployment in real-world applications is hindered by high sample complexity especially in the domains with sparse or noisy feedback signals and large or continuous state space. In many applications, an expert with additional domain knowledge (e.g., see ~\cite{ng1999policy,randlov2000shaping,brown2019machine,haug2018teaching,parameswaran2019interactive}) can help speed up the learning process by \emph{shaping the environment}. Here, the expert modifies the original environment such that the new environment is more learner-friendly, i.e., the learning agent can learn an optimal policy that maximizes discounted cumulative reward in a sample efficient manner (or within a practical amount of time). But the expert needs to ensure that the (near-)optimal policy learned in the shaped environment is near-optimal in the original environment as well, i.e., the optimal value function is preserved under the transformation. 

Potential-based reward shaping \cite{colombetti1994training,mataric1994reward,ng1999policy,wiewiora2003potential,devlin2012dynamic,asmuth2008potential} is a technique where the expert modifies the reward function of the learning environment, using a potential function so that the learning progress of the RL agent is accelerated. Existing works on reward shaping typically rely on an appropriate potential function, such as (approximate) optimal value function $V^*$, which is computationally challenging to estimate in large state spaces. When the feedback signals are noisy, even the shaped rewards cannot help the RL agent to escape the local optimal behaviors. Even though reward shaping is the most widely used technique to manipulate the environment, it can also be accomplished by changing the transition dynamics \cite{randlov2000shaping}. And there are instances where specific goals can only be achieved by changing the physics of the problem. 

To address the above limitations, we propose a new environment shaping framework, which is scalable to large or continuous state space, and is fundamentally different from the potential-based shaping approaches. In our framework, depending on the application setting, we have the flexibility of shaping: (i) only the reward, (ii) only the transition dynamics, or (iii) both reward and transition dynamics.

Our shaping framework relies on a value-irrelevant (formally defined later) state abstraction $\phi$ for the original environment $M$, and the corresponding (optimal) abstract policy $\pi_\phi$. This still requires domain knowledge; however, this is less knowledge (or maybe easier to acquire) than fully specifying a potential function. Given access to $(\phi, \pi_\phi)$, the expert employs the following pipeline: (a) first, shape the reward and/or transition dynamics in the abstract state space and construct an intermediate MDP $M_\phi$ by leveraging techniques from  \cite{ma2019policy,rakhsha2020policy} (b) then, from $M_\phi$, obtain the final shaped environment by appropriate \emph{lifting} (formalized later). 

Our main contributions include:
\begin{enumerate}[I]
\item We propose a novel framework for environment shaping in the domains with large or continuous state space and sparse or noisy reward signals. 
%We have covered all three cases of shaping: (i) reward only (ii) transition dynamics only and (iii) both reward transition dynamics.
\item Under mild technical conditions, we prove that the optimal policy in the shaped environment maintains its optimality in the original environment. The optimality guarantee depends on the abstraction quality, and other MDP related quantities.
\item We demonstrate the advantages of our shaping framework compared to the potential-based reward shaping methods on domains with sparse and noisy reward and large state space.
\end{enumerate}

Our techniques can be extended to another related problem of training-time adversarial attacks in RL, where existing algorithmic efforts are limited to finite state space, and do not scale when state space is continuous or large \cite{ma2019policy,rakhsha2020policy,zhang2020adaptive}. This extension is possible due to the perspective that policy teaching and attacking are mathematically equivalent. Here, the expert modifies the learning environment to induce desirable behavior in the RL agent \cite{zhang2008value}.

% \note{1 page content for Introduction. Finish this section by by page 1.5.}
%\note{1 page content for Introduction (excluding Abstract and Related subsection)}

%%%%%%%%%%%%%%%%%%%%%%%%%%%%%%%%%%%%%%%%%%%%%%%%%%%%%%%%%% Related Work
%\input{5_relatedwork}

%%%%%%%%%%%%%%%%%%%%%%%%%%%%%%%%%%%%%%%%%%%%%%%%%%%%%%%%%% MODEL
%!TEX root = main.tex
%%%%%%%%%%%%%%%%%%%%%%%%%%%%%%%%%%%%%%%%%%%%%%%%%%%%%%%%%
%%%%%%%%%%%%%%%%%%%%%%%%%%%%%%%%%%%%%%%%%%%%%%%%%%%%%%%%%%
\section{Problem Statement}\label{sec:model}

In this section, we formalize the problem of shaping an RL agent's learning environment such that its learning process can be accelerated. 

\paragraph{Setup.} 

Consider a learning environment represented by an MDP $M = \brr{\mathcal{S}, \mathcal{A}, T, R, \mu_0, \gamma}$. The state and action spaces are denoted by $\mathcal{S}$ and $\mathcal{A}$ respectively. Here, $\mathcal{S}$ is either finite or continuous, but $\mathcal{A}$ is finite. The transition function $T: \mathcal{S} \times \mathcal{S} \times \mathcal{A} \rightarrow \bss{0,1}$ is a probability mass function (or density function when $\mathcal{S}$ is continuous). The probability of landing in a state $s' \in \mathcal{S}$ (or any state in the region $\mathcal{S}' \subseteq \mathcal{S}$ when $\mathcal{S}$ is continuous) by taking action $a$ from state $s$, induced by $T$, is given by:
\begin{align}
\mathbb{P}_\mathrm{disc} \bss{s_{t+1} = s' \mid s_t = s, a_t = a ; T} ~=~& T\brr{s' \mid s , a} \text{ or } \label{transition-probs-1} \\
\mathbb{P}_\mathrm{cts} \bss{s_{t+1} \in \mathcal{S}' \mid s_t = s, a_t = a ; T} ~=~& \int_{\mathcal{S}'} T\brr{s' \mid s , a} ds' , \label{transition-probs-2}
\end{align}
respectively for discrete and continuous $\mathcal{S}$. The underlying reward function is given by $R: \mathcal{S} \times \mathcal{A} \rightarrow \bss{0,R_{\mathrm{max}}}$. Here $\gamma \in \brr{0,1}$ is the discounting factor, and $\mu_0: \mathcal{S} \rightarrow \bss{0,1}$ induces an initial distribution over the state space $\mathcal{S}$ (in similar way as Eqs.~\eqref{transition-probs-1}~and~\eqref{transition-probs-2}).  

We denote a policy $\pi: \mathcal{S} \rightarrow \Delta \brr{\mathcal{A}}$ as a mapping from a state to probability distribution over the action space. For any policy $\pi$, the state value function $V^\pi_M$ and the action value function $Q^\pi_M$ in the MDP $M$ are defined as follows respectively: 
\[
V^\pi_M \brr{s} = \expct{\sum_{t=0}^\infty \gamma^t r_t \mid s_0 =s, T, \pi} \quad \text{and} \quad Q^\pi_M \brr{s,a} = \expct{\sum_{t=0}^\infty \gamma^t r_t \mid s_0 =s, a_0 = a, T, \pi}. 
\]
Further the optimal value functions are given by $V^*_M\brr{s} = \max_\pi V^\pi_M \brr{s}$ and $Q^*_M \brr{s,a} = \max_\pi Q^\pi_M \brr{s,a}$. There always exists a deterministic stationary policy $\pi$ that achieves the optimal value function simultaneously for all $s \in \mathcal{S}$ \cite{puterman2014markov}, and we denote all such optimal policies by $\pi^*_M: \mathcal{S} \rightarrow \mathcal{A}$. We say that a policy $\pi$ is $\epsilon_{\mathrm{opt}}$-near optimal policy for the MDP $M$, if it satisfies the following condition:
\begin{equation}
\label{eq:near-opt-policy}
\sup_s \abs{V^*_M \brr{s} - V^{\pi}_M \brr{s}} ~\leq~  \epsilon_{\mathrm{opt}} .
\end{equation}

\paragraph{Objective.} 

When the state space is extremely large or continuous, and the reward function is sparse or noisy, learning a near optimal policy is computationally expensive. We aim to accelerate the learning of near optimal behavior, by shaping the reward function and/or the transition dynamics of the underlying MDP $M$. Formally, we want to construct a new shaped MDP $\overline{M} = \brr{\mathcal{S}, \mathcal{A}, \overline T, \overline R, {\overline{\mu}}_0, \gamma}$ from $M$ such that:
\begin{enumerate}
    \item any $\epsilon_\mathrm{opt}$-near optimal policy $\pi_{\overline{M}}: \mathcal{S} \rightarrow \Delta \brr{\mathcal{A}}$ of $\overline{M}$ satisfies the following condition:  
    \begin{equation}
        \label{eq:main-teach-obj}
        \sup_s \abs{V^*_M \brr{s} - V^{\pi_{\overline{M}}}_M \brr{s}} ~\leq~ c \cdot \epsilon_\mathrm{opt} ,
    \end{equation}
    for some $c > 0$, i.e., $\pi_{\overline{M}}$ is $c \cdot \epsilon_\mathrm{opt}$-near optimal policy for $M$, and
    \item learning a near-optimal policy $\pi_{\overline{M}}: \mathcal{S} \rightarrow \Delta \brr{\mathcal{A}}$ in the MDP $\overline{M}$ is (computationally) easier than learning a near optimal behavior in the original MDP $M$.
\end{enumerate}

\section{Our Approach}\label{sec:algorithm}

In this section, we introduce our environment shaping framework for easing the learning process of an RL agent. Here, we have the flexibility of modifying both the reward function and the transition dynamics of the learning environment. In Section~\ref{sub-sec:high-level}, we provide a high-level overview of our shaping framework and the domain knowledge required. Then, in Sections~\ref{sec:mphi}~and~\ref{sec:barm}, we get into the technical details of the pipeline and provide a theoretical guarantee on our overall objective.
%(\emph{cf.} Theorem~\ref{thm:main-change-T}).

\begin{figure}[htp]
    \centering
    \includegraphics[height=3cm]{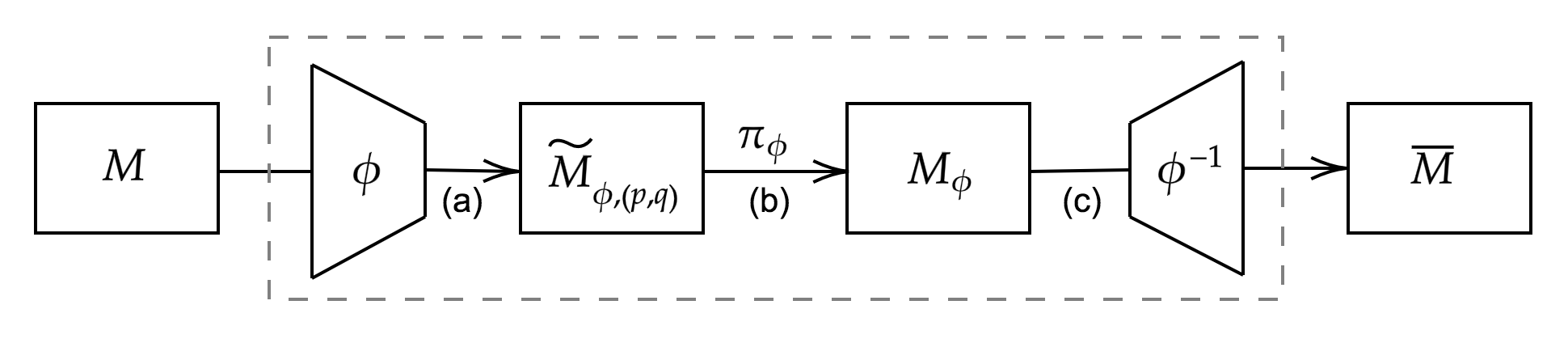}
    \caption{An illustration of our environment shaping framework: it takes the learning environment $M$ as input, and outputs the shaped environment $\overline{M}$. The framework uses an abstraction $\phi$, and a high-level policy $\pi_\phi$ as prior knowledge. The steps (a) and (b) of the pipeline are explained in Section~\ref{sec:mphi}, and step (c) in Section~\ref{sec:barm}.}
    \label{fig:shaping-framework}
\end{figure}

\subsection{High-level Ideas}
\label{sub-sec:high-level}

Here, we provide the high-level ideas of our environment shaping framework. Our approach leverages some of the techniques from state abstraction literature~\cite{givan2003equivalence,ravindran2004approximate,li2006towards,abel2016near}. In particular, we rely on compressed knowledge of the learning environment in the form of a high-level near-optimal policy to facilitate the RL agent's learning process. In contrast to potential-based shaping techniques, which require access to an appropriate potential function such as the exact/approximate optimal value function of the original learning environment, the domain knowledge demanded by our framework is practically feasible to obtain.

To formally explain the key ideas of our abstraction based shaping framework and the domain knowledge required for us, we begin by defining the following notions of approximate abstractions:
\begin{definition}[\cite{abel2016near,nanjiangnotes}]
Let $\mathcal{X}_\phi$ be a finite set. Given an MDP $M = \brr{\mathcal{S}, \mathcal{A}, T, R, \mu_0, \gamma}$ and state abstraction $\phi: \mathcal{S} \rightarrow \mathcal{X}_\phi$, we define three types of abstractions as follows:
\begin{enumerate}[(i)]
    \item $\phi$ is $\brr{\epsilon_R , \epsilon_T}$-approximate model irrelevant if $\forall{s_1, s_2 \in \mathcal{S}}$ where $\phi\brr{s_1} = \phi\brr{s_2}$, we have, $\forall{a \in \mathcal{A}}$:
    \begin{equation}
        \label{eq:model-inv}
        \abs{R\brr{s_1, a} - R\brr{s_2, a}} \leq \epsilon_R \quad \text{and} \quad \sum_{x' \in \mathcal{X}_\phi} \abs{T\brr{x' \mid s_1, a} - T\brr{x' \mid s_2, a}} \leq \epsilon_T ,
    \end{equation}
    where $T\brr{x' \mid s, a} := \int_\mathcal{S} T\brr{s' \mid s, a} \one\bcc{\phi\brr{s'} = x'} ds'$ for $x' \in \mathcal{X}_\phi$, $s \in \mathcal{S}$, and $a \in \mathcal{A}$.  
    \item $\phi$ is $\epsilon_{Q^*}$-approximate $Q^*$-irrelevant if there exists an abstract state-action value function $f_\phi: \mathcal{X}_\phi \times \mathcal{A} \rightarrow \mathbb{R}$ such that the following condition holds:
    \begin{equation}
        \label{eq:qval-inv}
        \sup_{s,a} \abs{Q^*_M \brr{s,a} - f_\phi \brr{\phi\brr{s},a}} ~\leq~ \epsilon_{Q^*} . 
    \end{equation}
    \item $\phi$ is $\epsilon_{V^*}$-approximate $V^*$-irrelevant if there exists an abstract policy $\pi_\phi: \mathcal{X}_\phi \rightarrow \mathcal{A}$ such that the following condition holds: 
    \begin{equation}
    \label{eq:given-policy-phi}
    \sup_s \abs{V^*_M \brr{s} - V^{\bss{\pi_{\phi}}_M}_M \brr{s}} ~\leq~ \epsilon_{V^*} ,
    \end{equation}
    where the lifted policy $\bss{\pi_{\phi}}_M: \mathcal{S} \rightarrow \mathcal{A}$ is defined as $\bss{\pi_{\phi}}_M \brr{s} := \pi_{\phi}\brr{\phi\brr{s}}$.
\end{enumerate}
\end{definition}

Formally, as a prior knowledge, our shaping framework requires access to an $\epsilon_{V^*}$-approximate $V^*$-irrelevant abstraction $\phi: \mathcal{S} \rightarrow \mathcal{X}_\phi$ (where $\mathcal{X}_\phi$ is a discrete set), and a corresponding policy $\pi_\phi: \mathcal{X}_\phi \rightarrow \mathcal{A}$ that satisfies the condition~\eqref{eq:given-policy-phi}. In Section~\ref{sec:practocal-changes}, we discuss the ways of obtaining such an abstraction, policy pair $\brr{\phi , \pi_\phi}$. Given this pair as an input, our abstraction-based shaping framework consists of the following two key steps:
\begin{enumerate}[(i)]
    \item First, starting from the tuple $\brr{M, \phi , \pi_\phi}$, we construct an intermediate MDP $M_\phi = \brr{\mathcal{X}_\phi, \mathcal{A}, T_\phi, R_\phi, \brr{\mu_0}_\phi, \gamma}$ such that the policy $\pi_\phi$ is the unique optimal policy for $M_\phi$. In addition, we ensure that $\pi_\phi$ is $\delta$ better than any other policy in the MDP $M_\phi$. This step makes our shaping technique fundamentally different from the potential-based shaping methods. We formalize this construction in Section~\ref{sec:mphi}.
    \item Then, using the intermediate MDP $M_\phi$, and the abstraction $\phi$, we construct our target MDP $\overline M$ such that the mapping $\phi$ is model irrelevant to $\overline M$. From this construction, we can show that any $\epsilon_\mathrm{opt}$-near optimal policy in $\overline M$ is also near optimal in the original MDP $M$. This construction is explained in Section~\ref{sec:barm}. 
\end{enumerate}
Our overall framework is illustrated in Figure~\ref{fig:shaping-framework}. 

\subsection{Construction of $M_\phi$}
\label{sec:mphi}

In this intermediate step, we aim to construct an MDP $M_\phi$, defined in the abstract space, that satisfies certain \emph{technical conditions} formally explained below. These conditions are essential to preserving the optimality behavior between the original learning environment and our final shaped environment. We note that this intermediate step along with these conditions distinguishes our shaping framework from the potential-based shaping methods.    

In order to formalize the preferred properties of the MDP $M_\phi$, we first introduce some distance measures between the components of the MDP $M_\phi$ and the original MDP $M$, for $p,q \geq 0$:
\begin{align*}
D_{p,q} \brr{T, {T}_\phi} ~:=~& \bss{\int_\mathcal{S} \sum_a \bcc{\bss{\int_{\mathcal{S}} \abs{T\brr{s' \mid s,a} - \frac{{T}_\phi \brr{\phi\brr{s'} \mid \phi\brr{s},a}}{\int_\mathcal{S} \one \bcc{\phi\brr{\bar s} = \phi \brr{s'}} d\bar{s}}}^q ds'}^{\frac{1}{q}}}^p ds}^{\frac{1}{p}} , \\
D_p \brr{R, {R}_\phi} ~:=~& \bss{\int_\mathcal{S} \sum_a {\abs{R\brr{s,a} - {R}_\phi \brr{\phi\brr{s},a}}}^p ds}^{\frac{1}{p}} , \text{ and } \\
D_p \brr{\mu_0, \brr{\mu_0}_\phi} ~:=~& \bss{\int_\mathcal{S} {\abs{\mu_0\brr{s} - \brr{\mu_0}_\phi \brr{\phi\brr{s}}}}^p ds}^{\frac{1}{p}} ,
\end{align*}
where ${T}_\phi$, ${R}_\phi$, and $\brr{\mu_0}_\phi$ belong to the following sets:
\begin{align*}
\mathcal{T}_\phi ~:=~& \bcc{{T}_\phi:\mathcal{X}_\phi \times \mathcal{X}_\phi \times \mathcal{A} \rightarrow \bss{0,1} \text{ s.t. } {T}_\phi \brr{\cdot \mid x,a} \in \Delta^{\abs{\mathcal{X}_\phi}}, \forall{x \in \mathcal{X}_\phi , a \in \mathcal{A}}} , \\
\mathcal{R}_\phi ~:=~& \bcc{{R}_\phi:\mathcal{X}_\phi \times \mathcal{A} \rightarrow \mathbb{R}} , \text{ and } \\
\brr{\mathcal{P}_0}_\phi ~:=~& \bcc{\brr{\mu_0}_\phi:\mathcal{X}_\phi  \rightarrow \bss{0,1} \text{ s.t. } \brr{\mu_0}_\phi \brr{\cdot} \in \Delta^{\abs{\mathcal{X}_\phi}}}
\end{align*}
respectively. Specifically, we are interested in the case where $p \to \infty$, and $q=1$. These distance measures quantify the level of modification done to the original learning environment. Based on the above definitions, we formalize the properties that we seek to ensure while constructing the MDP $M_\phi$. Given $\brr{M,\phi, \pi_\phi}$, we aim to construct a new MDP $M_\phi = \brr{\mathcal{X}_\phi, \mathcal{A}, T_\phi, R_\phi, \brr{\mu_0}_\phi, \gamma}$ such that the following conditions holds:
\begin{enumerate}[(i)]
    \item $D_{p,q} \brr{T, {T}_\phi}$, $D_p \brr{R, {R}_\phi}$, and $D_p \brr{\mu_0, \brr{\mu_0}_\phi}$ are minimized, i.e., we enforce the new MDP $M_\phi$ to not to change too much from the original MDP $M$.
    \item $M_\phi$ has a unique deterministic optimal policy $\pi^*_{M_\phi}$, and it satisfies:
    \begin{enumerate}
        \item $\pi_\phi \brr{x} = \pi^*_{M_\phi} \brr{x} = \argmax_\pi V^\pi_{M_\phi}\brr{x}$ for all $x \in \mathcal{X}_\phi$.
        \item any other deterministic policy in $M_\phi$ cannot be $\delta$-near optimal in the following sense:
        \begin{equation}
        \label{eq:unique-optimal-phi}
        Q^*_{M_\phi} \brr{x, \pi_\phi\brr{x}} ~\geq~ Q^*_{M_\phi} \brr{x, a'} + \delta , \quad \forall{a' \neq \pi_\phi\brr{x} , x \in \mathcal{X}_\phi} .
        \end{equation}
    \end{enumerate}
\end{enumerate}

%\paragraph{Approach.}
% \begin{align*}
% D_{\infty,1} \brr{T, {T}_\phi} ~:=~& \sup_{s,a} \int_{\mathcal{S}} \abs{T\brr{s' \mid s,a} - \frac{{T}_\phi \brr{\phi\brr{s'} \mid \phi\brr{s},a}}{\int_\mathcal{S} \one \bcc{\phi\brr{\bar s} = \phi \brr{s'}} d\bar{s}}} ds' , \\
% D_\infty \brr{R, {R}_\phi} ~:=~& \sup_{s,a} \abs{R\brr{s,a} - {R}_\phi \brr{\phi\brr{s},a}} , \text{ and } \\
% D_\infty \brr{\mu_0, \brr{\mu_0}_\phi} ~:=~& \sup_{s,a} \abs{\mu_0\brr{s} - \brr{\mu_0}_\phi \brr{\phi\brr{s}}} . 
% \end{align*}

Now, we sketch a way of obtaining such an MDP $M_\phi$. First, we construct an intermediate MDP ${\widetilde{M}}_{\phi,\brr{p,q}} = \brr{\mathcal{X}_\phi, \mathcal{A}, T_{\phi,\brr{p,q}}, R_{\phi,p}, \brr{\mu_0}_{\phi,p}, \gamma}$ as follows:
\begin{align*}
T_{\phi,\brr{p,q}} ~\gets~& \argmin_{{T}_\phi \in \mathcal{T}_\phi} D_{p,q} \brr{T, {T}_\phi} , \\
R_{\phi,p} ~\gets~& \argmin_{{R}_\phi \in \mathcal{R}_\phi} D_p \brr{R, {R}_\phi}, \text{ and } \\
\brr{\mu_0}_{\phi,p} ~\gets~& \argmin_{\brr{\mu_0}_\phi \in \brr{\mathcal{P}_0}_\phi} D_p \brr{\mu_0, \brr{\mu_0}_\phi} . \end{align*}
% Even though, we prefer the MDP $M_{\phi,\brr{\infty,1}}$, in practice, it might be easy construct $M_{\phi,\brr{2,2}}$, as follows, $\forall{x,x' \in \mathcal{X}_\phi, a \in \mathcal{A}}$:
% \begin{align}
% T_{\phi,\brr{2,2}} \brr{x' \mid x,a} ~=~& \frac{\int_{\mathcal{S}} \int_{\mathcal{S}} T \brr{s' \mid s,a} \one \bcc{\phi\brr{s'} = x'} \one \bcc{\phi\brr{s} = x} ds' ds}{\int_{\mathcal{S}} \one \bcc{\phi\brr{s} = x} ds} ,
% \label{eq:uniform-avg-transition} \\
% R_{\phi,2} \brr{x,a} ~=~& \frac{\int_{\mathcal{S}} R \brr{s,a} \one \bcc{\phi\brr{s} = x} ds}{\int_{\mathcal{S}} \one \bcc{\phi\brr{s} = x} ds} , \text{ and }
% \label{eq:uniform-avg-reward} \\
% \brr{\mu_0}_{\phi,2} \brr{x} ~=~& \frac{\int_{\mathcal{S}} \mu_0 \brr{s} \one \bcc{\phi\brr{s} = x} ds}{\int_{\mathcal{S}} \one \bcc{\phi\brr{s} = x} ds} .
% \label{eq:uniform-avg-init} 
% \end{align}
After obtaining ${\widetilde{M}}_{\phi,\brr{p,q}}$, we can construct our target $M_\phi$ by leveraging techniques from~\cite{ma2019policy,rakhsha2020policy,zhang2020adaptive}.

% \begin{enumerate}
%     \item $M_{\phi,\brr{p,q}}^{\mathrm{atk:R}}$ by attacking only the reward function $R_{\phi,p}$.
%     \item $M_{\phi,\brr{p,q}}^{\mathrm{atk:T}}$ by attacking only the transition function $T_{\phi,\brr{p,q}}$.
%     \item $M_{\phi,\brr{p,q}}^{\mathrm{atk:R,T}}$ by attacking both $R_{\phi,p}$, and $T_{\phi,\brr{p,q}}$.
% \end{enumerate}
% Note that $M_{\phi,\brr{p,q}}^{\mathrm{atk:R}}$, $M_{\phi,\brr{p,q}}^{\mathrm{atk:T}}$, and $M_{\phi,\brr{p,q}}^{\mathrm{atk:R,T}}$ (specifically with $p = \infty$, and $q=1$) satisfy our requirements. 

\subsection{Construction of $\overline M$}
\label{sec:barm}

% In this case, a meaningful choice for $M_\phi$ is
% \[
% M_\phi ~\gets~ M_{\phi,\brr{p,q}}^{\mathrm{atk:R}},~M_{\phi,\brr{p,q}}^{\mathrm{atk:T}},~\text{or}~M_{\phi,\brr{p,q}}^{\mathrm{atk:R,T}} .
% \]

Now, we turn to our final step of constructing our target MDP $\overline M$, which can be viewed as a decompression step. In this case, we assume that both the reward and the transition dynamics can be modified. Then, given the abstraction $\phi$, and the MDP $M_\phi$ constructed in Section~\ref{sec:mphi}, we construct a shaped MDP $\overline{M} = \brr{\mathcal{S}, \mathcal{A}, \overline T, \overline R, \overline \mu_0, \gamma}$ as follows:
\begin{enumerate}
\item $\overline{T} \brr{s' \mid s,a} = \frac{T_\phi \brr{\phi\brr{s'} \mid \phi\brr{s},a}}{\int_{\mathcal{S}} \one \bcc{\phi\brr{\tilde{s}} = \phi\brr{s'}} d\tilde{s}}$, for all $s' \in \mathcal{S}, s \in \mathcal{S}, a \in \mathcal{A}$
\item $\overline{R} \brr{s,a} = R_\phi \brr{\phi\brr{s},a}$, for all $s \in \mathcal{S}, a \in \mathcal{A}$
\item $\overline{\mu}_0 \brr{s} = \brr{\mu_0}_\phi \brr{\phi\brr{s}}$, for all $s \in \mathcal{S}$
\end{enumerate}

Finally, we want to ensure that our teaching objective, mentioned in Section~\ref{sec:model}, can be attained by this shaped environment $\overline M$. To this end, the following theorem shows that the agent's policy learnt in the shaped environment $\overline M$ preserves near-optimal behavior in the original environment $M$. Then, in Section~\ref{sec:experiments}, we empirically demonstrate that our shaping framework indeed helps in accelerating the learning process of the RL agent. 

\begin{theorem}
\label{thm:main-change-T}
Any $\epsilon_\mathrm{opt}$-near optimal policy $\pi_{\overline{M}}: \mathcal{S} \rightarrow \Delta\brr{\mathcal{A}}$ of $\overline{M}$, satisfies the following condition:  
\begin{equation}
    \label{eq:teaching-objective}
    \sup_s \abs{V^*_M \brr{s} - V^{\pi_{\overline{M}}}_M \brr{s}} ~\leq~ \epsilon'_\mathrm{opt} ~:=~ \epsilon_{V^*} + \frac{R_{\mathrm{max}}}{1 - \gamma} \cdot \frac{\epsilon_\mathrm{opt}}{\delta} .
\end{equation}
When $\epsilon_{V^*} \leq \frac{R_{\mathrm{max}}}{1 - \gamma} \cdot \frac{\epsilon_\mathrm{opt}}{\delta}$, we attain our teaching objective Eq.~\eqref{eq:main-teach-obj} with $c = \frac{R_{\mathrm{max}}}{1 - \gamma} \cdot \frac{2}{\delta}$.
\end{theorem}

\begin{proof}
%Lemma~\ref{lemma:abstraction-relationship} 
For the target MDP $\overline{M}$, we first show that $\bss{\pi_\phi}_{\overline{M}}$ is the unique optimal policy. One can easily observe that, by construction, $\phi$ is a model-irrelevant abstraction for $\overline{M}$ ($\epsilon_R = 0$, and $\epsilon_T = 0$ in Eq.~\eqref{eq:model-inv}). Then, from Lemma~1 (in the Appendix of the supplementary material), $\phi$ is also a $Q^*$-irrelevant abstraction for $\overline{M}$ ($\epsilon_{Q^*} = 0$ in Eq.~\eqref{eq:qval-inv}), i.e., we have 
\[
Q^*_{\overline{M}} \brr{s,a} ~=~ Q^*_{M_\phi} \brr{\phi\brr{s},a} , \forall{s \in \mathcal{S}, a \in \mathcal{A}} .
\]
And since $\pi_\phi$ is the unique optimal policy derived from $Q^*_{M_\phi}$, we can conclude that $\bss{\pi_\phi}_{\overline{M}}$ is the unique optimal policy for $\overline{M}$. 

%Lemma~\ref{near-opt-lemma}
In addition, using the Lemma~5 (in the Appendix of the supplementary material), we can show that any $\epsilon_\mathrm{opt}$-near optimal policy $\pi_{\overline{M}}: \mathcal{S} \rightarrow \Delta\brr{\mathcal{A}}$ of $\overline{M}$, satisfies the following condition:
\begin{equation}
\label{eq:tv-close}
\sup_s \norm{\pi_{\overline{M}}\brr{\cdot \mid s} -  \bss{\pi_\phi}_{\overline{M}}\brr{\cdot \mid s}}_1 ~\leq~ \frac{\epsilon_\mathrm{opt}}{\delta} .
\end{equation}

Finally, consider:
\begin{align*}
\sup_s \abs{V^*_M \brr{s} - V^{\pi_{\overline{M}}}_M \brr{s}} ~\stackrel{(i)}{=}~& \sup_s \abs{V^*_M \brr{s} - V^{\bss{\pi_\phi}_M}_M \brr{s} + V^{\bss{\pi_\phi}_{\overline M}}_M \brr{s} - V^{\pi_{\overline{M}}}_M \brr{s}} \\
~\stackrel{(ii)}{\leq}~& \sup_s \abs{V^*_M \brr{s} - V^{\bss{\pi_\phi}_M}_M \brr{s}} + \sup_s \abs{V^{\bss{\pi_\phi}_{\overline M}}_M \brr{s} - V^{\pi_{\overline{M}}}_M \brr{s}} \\
~\stackrel{(iii)}{\leq}~& \epsilon_{V^*} + \sup_s \abs{V^{\bss{\pi_\phi}_{\overline M}}_M \brr{s} - V^{\pi_{\overline{M}}}_M \brr{s}} \\
~\stackrel{(iv)}{\leq}~& \epsilon_{V^*} + \frac{R_{\mathrm{max}}}{1 - \gamma} \sup_s \norm{\pi_{\overline{M}}\brr{\cdot \mid s} -  \bss{\pi_\phi}_{\overline{M}}\brr{\cdot \mid s}}_1 \\
~\stackrel{(v)}{\leq}~& \epsilon_{V^*} + \frac{R_{\mathrm{max}}}{1 - \gamma} \cdot \frac{\epsilon_\mathrm{opt}}{\delta} ~=~ \epsilon'_\mathrm{opt} , \end{align*}
%\ref{lemma:different-policy}
where $(i)$ is due to the fact that $\bss{\pi_\phi}_{\overline{M}} = \bss{\pi_\phi}_{M}$, $(ii)$ is by triangle inequality, $(iii)$ is from Eq.~\eqref{eq:given-policy-phi}, $(iv)$ is due to Lemma~2 (in the Appendix of the supplementary material), and $(v)$ is by Eq.~\eqref{eq:tv-close}.        
\end{proof}

\section{Practical Considerations}
\label{sec:practocal-changes}
In this section, we discuss some practical aspects of our environment shaping framework. 

\paragraph{Constructing Abstractions.}

Here, we briefly discuss the ways of obtaining the abstraction, and high-level policy pair $\brr{\phi,\pi_\phi}$, which is required as a prior knowledge to our shaping framework. One way of obtaining such knowledge is via careful feature engineering with the help of a domain expert. For example, when the ground state space is raw pixels acquired through noisy sensors, an expert can focus only on the relevant information that is required for the decision making task. One can also utilize the existing computational approaches, such as~\cite{jong2005state,abel2019state, sinclair-2019}, to automatically construct meaningful abstractions.   

\paragraph{Reward-only Shaping.}

In the case of shaping via only modifying the reward function, the shaped MDP $\overline{M}$ differs from the original MDP $M$ only in its reward function $R$. In particular, given the abstraction $\phi$, and the MDP $M_\phi$ constructed in Section~\ref{sec:mphi}, we construct a reward-only shaped MDP $\widehat{M} = \brr{\mathcal{S}, \mathcal{A}, \widehat T, \widehat R, \widehat \mu_0, \gamma}$ as follows:
\begin{enumerate}
\item $\widehat{T} \brr{s' \mid s,a} = {T} \brr{s' \mid s,a}$, and $\widehat{\mu}_0 \brr{s} = \mu \brr{s}$, for all $s' \in \mathcal{S}, s \in \mathcal{S}, a \in \mathcal{A}$
\item $\widehat{R} \brr{s,a} = R_\phi \brr{\phi\brr{s},a}$, for all $s \in \mathcal{S}, a \in \mathcal{A}$
\end{enumerate}

The following theorem, similar to Theorem~\ref{thm:main-change-T}, provides a guarantee on preserving the optimality behavior in the reward-only shaped environment.

\begin{theorem}
\label{thm:reward-only-shaping}
We assume the following:
\[
\sup_{s,a} ~ \norm{T\brr{\cdot \mid s,a} - \overline T\brr{\cdot \mid s,a}}_1 ~\leq~ \beta_T . 
\]
Then, any $\epsilon_\mathrm{opt}$-near optimal policy $\pi_{\widehat{M}}: \mathcal{S} \rightarrow \Delta\brr{\mathcal{A}}$ of $\widehat{M}$, satisfies the following condition:  
\begin{equation}
    \label{eq:teaching-objective-new}
    \sup_s \abs{V^*_M \brr{s} - V^{\pi_{\widehat{M}}}_M \brr{s}} ~\leq~ \epsilon_{V^*} + \frac{R_{\mathrm{max}}}{1 - \gamma} \cdot \frac{1}{\delta} \cdot \bcc{\frac{2 \cdot \gamma \cdot \beta_T \cdot R_\phi^{\mathrm{max}}}{\brr{1-\gamma}^2} + \epsilon_\mathrm{opt}} ,
\end{equation}
where $R_\phi^{\mathrm{max}} = \max_{x,a} \abs{R_\phi \brr{x,a}}$. 
\end{theorem}

\paragraph{Transition Dynamics-only Shaping.}
Here, we shape only the transition dynamics while keeping the reward function of the shaped MDP $\overline{M}$ is the same as that of the original MDP $M$. In particular, given the abstraction $\phi$, and the MDP $M_\phi$ constructed in Section~\ref{sec:mphi}, we construct a target/shaped MDP $\widehat{M} = \brr{\mathcal{S}, \mathcal{A}, \widehat T, \widehat R, \widehat \mu_0, \gamma}$ as follows:
\begin{enumerate}
\item $\widehat{T} \brr{s' \mid s,a} = \frac{T_\phi \brr{\phi\brr{s'} \mid \phi\brr{s},a}}{\int_{\mathcal{S}} \one \bcc{\phi\brr{\tilde{s}} = \phi\brr{s'}} d\tilde{s}}$, for all $s' \in \mathcal{S}, s \in \mathcal{S}, a \in \mathcal{A}$
\item $\widehat{R} \brr{s,a} = R \brr{s,a}$, and $\widehat{\mu}_0 \brr{s} = \mu \brr{s}$, for all $s \in \mathcal{S}, a \in \mathcal{A}$
% \item $\widehat{\mu}_0 \brr{s} = \mu \brr{s}$, for all $s \in \mathcal{S}$
\end{enumerate}

The following theorem, similar to Theorem~\ref{thm:main-change-T}, provides a guarantee on preserving the optimality behavior in the dynamics-only shaped environment.

\begin{theorem}
\label{thm:dynamics-only-shaping}
We assume the following:
\[
\sup_{s,a} ~ \abs{R\brr{s,a} - \overline R\brr{s,a}} ~\leq~ \beta_R \cdot \max \bcc{R_\mathrm{max}, R_\phi^{\mathrm{max}}} . 
\]
Then, any $\epsilon_\mathrm{opt}$-near optimal policy $\pi_{\widehat{M}}: \mathcal{S} \rightarrow \Delta\brr{\mathcal{A}}$ of $\widehat{M}$, satisfies the following condition:  
\begin{equation}
    \label{eq:teaching-objective-new}
    \sup_s \abs{V^*_M \brr{s} - V^{\pi_{\widehat{M}}}_M \brr{s}} ~\leq~ \epsilon_{V^*} + \frac{R_{\mathrm{max}}}{1 - \gamma} \cdot \frac{1}{\delta} \cdot \bcc{\frac{2 \cdot \beta_R \cdot \max \bcc{R_\mathrm{max}, R_\phi^{\mathrm{max}}}}{{1-\gamma}} + \epsilon_\mathrm{opt}} ,
\end{equation}
where $R_\phi^{\mathrm{max}} = \max_{x,a} \abs{R_\phi \brr{x,a}}$. 
\end{theorem}

The proofs of Theorems~\ref{thm:reward-only-shaping}~and~\ref{thm:dynamics-only-shaping} can be found in the Appendix of the supplementary material. 
%The proof technique is similar to the one used for proving Theorem~\ref{thm:reward-only-shaping}.

%%%%%%%%%%%%%%%%%%%%%%%%%%%%%%%%%%%%%%%%%%%%%%%%%%%%%%%%%% Experiments
%\clearpage
%!TEX root = main.tex
%%%%%%%%%%%%%%%%%%%%%%%%%%%%%%%%%%%%%%%%%%%%%%%%%%%%%%%%%
\section{Experimental Evaluation}\label{sec:experiments}
\vspace{-2mm}
% Next, we describe our experiments performed on the \textit{object gathering game} environment see Fig \ref{fig:chain.env} for the details. 

% In this section we describe our experiments which illustrate  whether our proposed environment shaping approach enables simple reinforcement learning algorithm e.g., policy gradient, to find (near-)optimal policy.

In this section, we empirically investigate whether our environment shaping approach enables simple reinforcement learning algorithm such as policy gradient, to find (near-)optimal policy under challenging setting with noisy and sparse feedback. 

\subsection{Environment Details}
\vspace{-2mm}
% We considered \textit{object gathering game} environment see Fig \ref{fig:chain.env} for the details. In this game, an agent starts from one of the corner cells of the environment. The goal of the agent is to collect the objects. Collecting a ``star" object gives a reward of $1.0$, a ``plus" provides $0.04$ reward, and collecting an object ``dot" provides either $0.1$ or $-0.1$ with equal probability. At initialization  object "star" always appears opposite corner of the agent's position, ``plus" appears in the four cells distance from the agent's initial position and ``dot" appears uniformly at random anywhere. 

We consider the \textit{object gathering game} environment, inspired by environments considered in recent works \cite{leibo17,raileanu18modeling,ghosh2020towards-aamas,DBLP:conf/nips/TschiatschekGHD19}, see Figure~\ref{fig:chain.env}. In this game, an agent starts from one of the corner cells. The goal of the agent is to collect the objects, so that to maximize its total return. By collecting a ``star" object, the agent receives a reward of $1.0$, a ``plus" object provides $0.04$ reward, and collecting a ``dot" object provides either $0.1$ or $-0.1$ reward with equal probability. At initialization, the ``star" object always appears opposite corner of the agent's position, ``plus" appears in the four cells distance from the agent's initial position, and ``dot" appears uniformly at random anywhere. After collecting the ``plus" object, it appears on the same location, and ``dot" object disappears after collecting it. The game's episode ends when the maximum time step $H=30$ is reached, or the agent has collected the ``star" object. In this game we have two actions \{``left", ``right"\}, i.e., $|\mathcal{A}|=2$. The state of the system is represented by the position of the agent and the positions of all three objects. Since we have the number of the cells to be $15$ and our state representation is described as the concatenation of the binary vector of the position of the agent and the position of each object, it will give us rise to $15^4$ total possible states, i.e., $|\mathcal{S}|=15^4$. The transition dynamics of the environment is defined as follows: with probability $0.99$, the actions succeed in navigating the agent to left or right, as shown with the arrows in Figure \ref{fig:chain.env}; and with probability $0.01$, the move to the opposite direction is performed. The rewards are discounted at $\gamma=0.99$.  The optimal policy $\pi^*$ is to go towards to the ``star"  object and collect it. 

% After collecting the object ``plus" it appears on the same location and object ``dot" disappears after collecting it. The episode of the game ends when the maximum time step $H=30$ is reached or the agent has collected an object ``star". In this game we have two actions \{``left", ``right"\}, $|\mathcal{A}|=2$. The state of the system is represented by the position of the agent and the positions of all the three objects on the cell. Since we have the number of the cells to be $15$ and our state representation is described as the concatenation of the binary vector of the position of the agent and the position of each object on the line, it will give us rise to $15^4$ total possible states, $|\mathcal{S}|=15^4$. The transition of the agent is as follows: with probability $0.99$, the actions succeed in navigating the agent to left or right as shown with the arrows (see Figure \ref{fig:chain.env}); with probability $0.01$  the move to the opposite direction is performed. Rewards are discounted with $\gamma=0.99$.  The optimal policy $\pi^*$ is to go towards to the ``star"  object and collect it. 

\begin{figure*}[t!]
	\begin{subfigure}{0.45\linewidth}
		\includegraphics[width=1.15\linewidth]{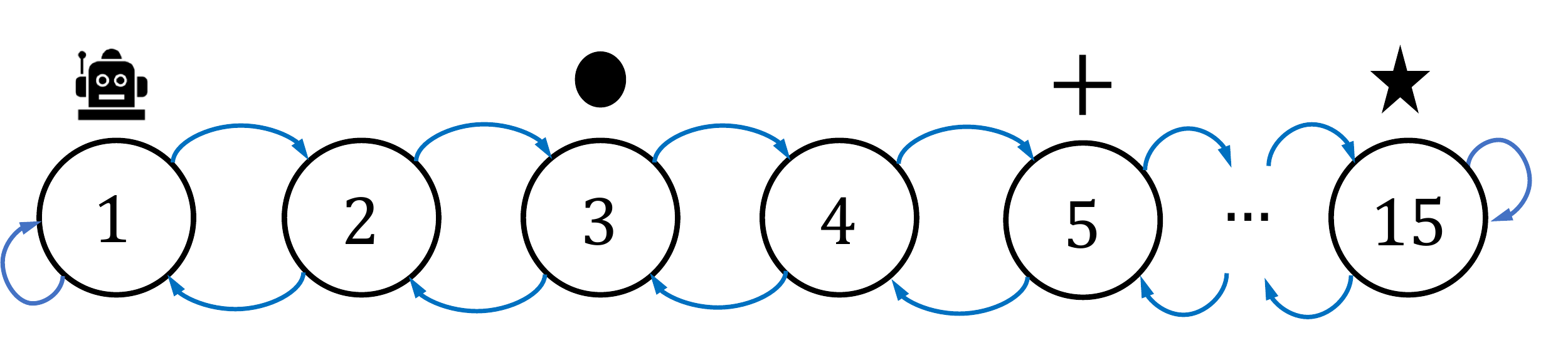}
    	\caption{Environment}
        \label{fig:chain.env}
	\end{subfigure}
	\quad \quad \quad
	\begin{subfigure}{0.45\linewidth}
		\includegraphics[height=.8\linewidth]{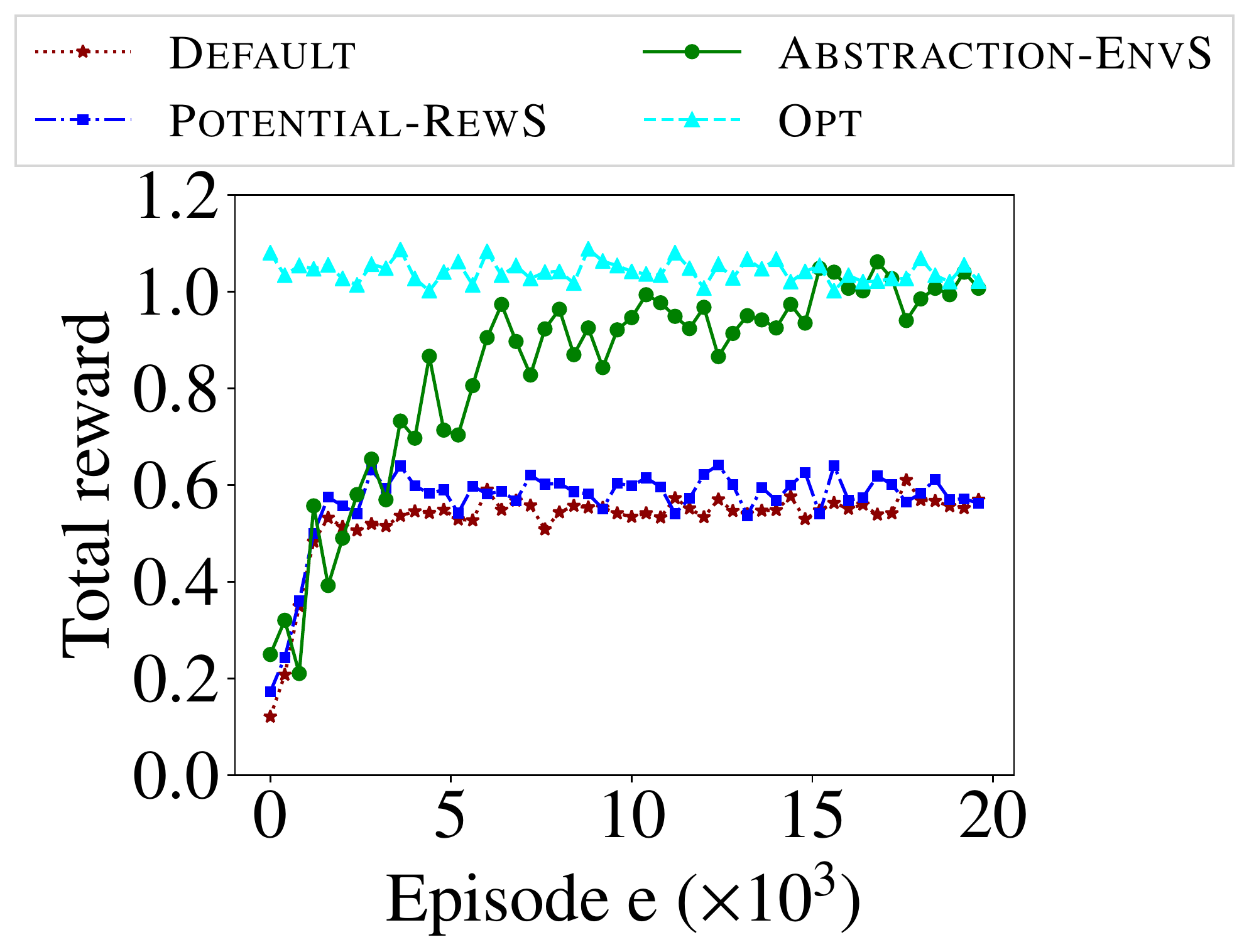}
		\vspace{-4mm}
    	\caption{Performance}
        \label{fig:results.1}
	\end{subfigure}
	\caption{(a) An example of the \textit{gathering game} environment with the actions $\mathcal{A}$=\{"left", "right"\} and the state space size $|\mathcal{S}|=15^4$. (b) Learning curves for \textit{gathering game} task, averaged across 30 runs of each algorithm with random initialization.}
\end{figure*}

%  Plot shows the convergence behavior of the policy as more episodes are presented to the agent.
% An example of the \textit{gathering game} environment with the actions $\mathcal{A}$=\{"left", "right"\} and state space $|\mathcal{S}|=6^4$.

%%%%%%%%%%%%%%%
%%%%%%%%%%%%%%%

% We evaluate performance of different learning algorithms on the environment presented in Figure \ref{fig:chain.env}. As a baseline we used policy gradient method to learn optimal policy.
%\paragraph{Baselines:} 
%\vspace{-3mm}
\subsection{Experimental Setup}
\vspace{-2mm}

% As a learning algorithm we use vanilla policy gradient method, where the policy network is implemented as a two-layer neural network with \texttt{relu} activation function and \texttt{softmax} on the last layer. 
% The location of the agent and locations of all the objects on the cells is represented as a one-hot encoding of a vector of length $15$, corresponding to the number of the cells. Hence the length of the input vector to the neural network is $4\cdot15=60$.

We use the vanilla policy gradient method as the agent's algorithm, where the policy network is implemented as a two-layer neural network with \texttt{relu} activation function and \texttt{softmax} on the last layer. A one-hot vector of length $15$, corresponding to the number of cells, represents the agent's location and the positions of all three objects in the environment. Hence the length of the input vector to the neural network is $4\times15=60$.

% The following algorithms were considered in the comparison:  (i) Training the agent using vanilla policy gradient algorithm on the original environment $M$, (ii) training the agent with potential-based reward shaping using value function proposed in \cite{ng1999policy}, and (iii) training the agent on shaped environment described in this paper.

\textbf{Methods evaluated.} The following algorithms were considered in the comparison: (i) default baseline with training on original $M$ (\textsc{Default}), (ii) baseline of potential-based reward shaping of $M$  (\textsc{Potential-RewS}) (see \cite{ng1999policy}), (iii) training on $\overline{M}$ shaped using our approach (\textsc{Abstraction-EnvS}), and (iv) optimal policy that follows object ``star" (\textsc{Opt}). 
%see Figure \ref{fig:results.1}.

% (i) Training the agent on the original environment $M$, (ii) training the agent with potential-based reward shaping using value function proposed in~\cite{ng1999policy}, and (iii) training the agent on the shaped environment $\overline{M}$ described in Section~\ref{sec:algorithm}.

% To use potential-based shaping approach we need to approximate value function  $V^{\pi}$ of the policy $\pi$. First, we run policy gradient algorithm up to $50000$ iterations and then we used resulting policy for the approximation of $V^{\pi}$.
% Value function of any given  policy $\pi$ can be estimated using \textit{Monte Carlo (MC)} prediction algorithm. In our case we used \textit{MC} prediction where the  number of samples is $100000$.

For the \textsc{Potential-RewS} approach, we need to approximate the optimal value function  $V_{M}^{*}$ of the original MDP $M$. To this end, we first run the policy gradient algorithm up to $50,000$ iterations, and then used resulting policy $\pi$ to approximate $V_{M}^{*}$ by $V_M^{\pi}$. To estimate the value function of the policy $\pi$, we used Monte Carlo (MC) prediction algorithm with $100,000$ samples. Given that we have domain knowledge of our problem, e.g., the information about the agent's position and `star" object's position is sufficient to learn the optimal behavior, our abstraction $\phi:\mathcal{S}\rightarrow \mathcal{X}_\phi$ can be defined as follows: we consider only the position of the agent and position of the ``star" object to encode our state space. This abstraction reduces the state-space from $15^4$ down to $15^2$, i.e.,  $|\mathcal{X}_\phi| = 15^2$. We construct the MDP $M_\phi$ by first obtaining ${\widetilde{M}}_{\phi,\brr{2,1}}$ and then using the rewards only attack as in \cite{rakhsha2020policy}. 

% \textcolor{red}{For the creation of the environment $M_\phi$ with the parameters $p=2$ and $q=1$ we used rewards only attack as described in \cite{rakhsha2020policy}, which guaranties, that $M_\phi$ will have $\delta=2$ unique optimal policy $\pi_\phi$. -- Rati, are you using reward only attack?}
% \textcolor{blue}{Yes, i used reward only attack with the norm $p=\infty$ and  margin $\delta=2$}

\vspace{-2mm}
%\paragraph{Experimental results}
\subsection{Results}
\vspace{-2mm}
% We evaluate the performance of different algorithms and report the results by averaging across $30$ runs and with random initialization for each run. Figure \ref{fig:results.1} shows the convergence of total reward of the episode to optimal total reward which is calculated using the policy following the ``star" object, e.g., optimal policy. As expected, training on the shaped environment outperforms two baselines.

We evaluate the performance of the four algorithms mentioned above and report the results by averaging across $30$ runs and with random initialization for each run. Figure~\ref{fig:results.1} shows the convergence of the total reward of the episode for all the algorithms, where the policy following the ``star" object attains the optimal value. As expected, training on the shaped environment outperforms two baselines.

% Since the training on the original environment, $M$ needs exponentially many episodes to learn optimal policy, it converges to local minimum e.g., collecting nearest object ``plus". This is because of the sparse reward problem (see \cite{2002-kakade-langford}). Since our environment provides noisy signals feedback because of the randomness of the ``dot" object, this makes it challenging to compute an appropriate potential function, and as a result the applicability of the potential-based reward shaping is limited in our settings.

Since the training on the original environment $M$ needs exponentially many episodes to learn optimal policy, it converges to a local minimum~\cite{2002-kakade-langford}, e.g., collecting the nearest object ``plus". Since our environment provides noisy signals feedback because of the randomness of the ``dot" object, it is challenging to compute an appropriate potential function. As a result, the applicability of the potential-based reward shaping is limited in our settings.

% The results on the Figure \ref{fig:results.1}  validate the performance
% guarantees that we proved in the previous sections. 

%%%%%%%%%%%%%%%%%%%%%%%%%%%%%%%%%%%%%%%%%%%%%%%%%%%%%%%%%% Conclusions
%!TEX root = main.tex
%%%%%%%%%%%%%%%%%%%%%%%%%%%%%%%%%%%%%%%%%%%%%%%%%%%%%%%%%
\vspace{-3mm}
\section{Conclusions}\label{sec:conclusion}
\vspace{-2mm}

In this work, we propose an abstraction based environment shaping framework to accelerate the learning process of an agent. Our framework is fundamentally different from the existing potential-based shaping methods. We theoretically prove that our environment shaping preserves the optimality behavior between the original and target MDPs, and empirically demonstrate its efficacy on tasks with sparse and noisy feedback signals.  

\clearpage
%\begin{ack}
\section*{Acknowledgments}
This work was supported in part by the European Research Council (ERC) under the European Union's Horizon 2020 research and innovation program (grant agreement n 725594 -- time-data), and the Swiss National Science Foundation (SNSF) under grant number 407540\_167319.
%\end{ack}

%\begin{ack}
%This work has received funding in part from the European Research Council (ERC) under the European Union's Horizon 2020 research and innovation program (grant agreement n 725594 - time-data), and the Swiss National Science Foundation (SNSF) under grant number 407540\_167319.
%\end{ack}

%This work was supported in part by the Swiss National Science Foundation (SNSF) under grant number 407540_167319

%%%%%%%%%%%%%%%%%%%%%%%%%%%%%%%%%%%%%%%%%%%%%%%%%%%%%%%%% BIB
%\input{7_impact}
\bibliography{main}
\bibliographystyle{unsrt}

%%%%%%%%%%%%%%%%%%%%%%%%%%%%%%%%%%%%%%%%%%%%%%%%%%%%%%%%%% APPENDIX
\clearpage
\appendix
{\allowdisplaybreaks
%!TEX root = main.tex
%%%%%%%%%%%%%%%%%%%%%%%%%%%%%%%%%%%%%%%%%%%%%%%%%%%%%%%%
\section{Auxiliary Lemmas}\label{sec.appendix.mmp.main}

\begin{lemma}[Theorem~2 from \cite{nanjiangnotes}]
\label{lemma:abstraction-relationship}
If $\phi$ is an $\brr{\epsilon_R , \epsilon_T}$-approximate model-irrelevant abstraction, then $\phi$ is also an approximate $Q^*$-irrelevant abstraction with approximation error $\epsilon_{Q^*} = \frac{\epsilon_R}{1-\gamma} + \frac{\gamma \epsilon_T R_\mathrm{max}}{2 \brr{1-\gamma}^2}$. Further, if $\phi$ is an $\epsilon_{Q^*}$-approximate $Q^*$-irrelevant abstraction, then $\phi$ is also an approximate $V^*$-irrelevant abstraction with approximation error $\epsilon_{V^*} = \frac{2 \epsilon_{Q^*}}{1-\gamma}$. 
\end{lemma}

\begin{lemma}[Lemma~A.1 from \cite{sun2018dual}]
\label{lemma:different-policy}
For any two policy $\pi$ and $\pi'$ (acting in the MDP $M$), we have
\[
\sup_s \abs{V^{\pi}_M \brr{s} - V^{\pi'}_M \brr{s}} ~\leq~ \frac{R_{\mathrm{max}}}{1 - \gamma} \sup_s \norm{\pi \brr{\cdot \mid s} -  \pi'\brr{\cdot \mid s}}_1 . 
\]
\end{lemma}

\begin{definition}[approximately-equivalent MDPs, adapted from \cite{even2003approximate}]
Suppose we have two MDPs $M_1 = \brr{\mathcal{S}, \mathcal{A}, T_1, R_1, \brr{\mu_0}_1, \gamma}$ and $M_2 = \brr{\mathcal{S}, \mathcal{A}, T_2, R_2, \brr{\mu_0}_2, \gamma}$, and rewards are bounded in $\bss{0,R_\mathrm{max}}$. We call $M_1$ and $M_2$ as $\brr{\beta_R, \beta_T}$ approximately-equivalent if the following holds:
\begin{align*}
\sup_{s \in \mathcal{S},a \in \mathcal{A}} ~ \abs{R_1\brr{s,a} - R_2\brr{s,a}} ~\leq~& \beta_R \cdot R_\mathrm{max} \\
\sup_{s \in \mathcal{S},a \in \mathcal{A}} ~ \norm{T_1\brr{\cdot \mid s,a} - T_2\brr{\cdot \mid s,a}}_1 ~\leq~& \beta_T .
\end{align*}
\end{definition}

\begin{lemma}
\label{lemma:robust-diff}
Suppose we have two $\brr{\beta_R, \beta_T}$ approximately-equivalent MDPs $M_1$ and $M_2$. Then, we have:
\[
\sup_s \abs{V^*_{M_1} \brr{s} - V^*_{M_2} \brr{s}} ~\leq~ \frac{\beta_R \cdot R_{\mathrm{max}}}{1 - \gamma} + \frac{\gamma \cdot \beta_T \cdot R_{\mathrm{max}}}{\brr{1 - \gamma}^2} . 
\]
\end{lemma}

\begin{proof}
Consider the Bellman optimality equation, for any $s$:
\[
V^*_{M} (s) ~=~ \max_a Q^*_{M} (s,a) ~=~ \max_a \bcc{R (s,a) + \gamma \int_\mathcal{S} T (s' \mid s,a) V^*_{M} (s') ds'}
\]
Then, for any $s$, we have:
\begin{align*}
& \abs{V^*_{M_1} (s) - V^*_{M_2} (s)} \\
=~& \abs{\max_a Q^*_{M_1} (s,a) - \max_{a'} Q^*_{M_2} (s,a')} \\
\leq~& \max_a \abs{Q^*_{M_1} (s,a) - Q^*_{M_2} (s,a)} \\
=~& \max_a \abs{\bs{R_1 (s,a) - R_2 (s,a)} + \gamma \int_\mathcal{S} \bs{T_1 (s' \mid s,a) V^*_{M_1} (s') - T_2 (s' \mid s,a) V^*_{M_2} (s')} ds'} \\
\leq~& \max_a \abs{R_1 (s,a) - R_2 (s,a)} + \gamma \max_a \abs{\int_\mathcal{S} \bs{T_1 (s' \mid s,a) V^*_{M_1} (s') - T_2 (s' \mid s,a) V^*_{M_2} (s')} ds'} \\
\leq~& \beta_R \cdot R_{\mathrm{max}} + \gamma \max_a \abs{\int_\mathcal{S} \bs{T_1 (s' \mid s,a) - T_2 (s' \mid s,a)} V^*_{M_1} (s') ds' - \int_\mathcal{S} T_2 (s' \mid s,a) \bs{V^*_{M_2} (s') - V^*_{M_1} (s')}ds'} \\
\leq~& \beta_R \cdot R_{\mathrm{max}} + \gamma \max_a \bcc{\abs{\int_\mathcal{S} \bs{T_1 (s' \mid s,a) - T_2 (s' \mid s,a)} V^*_{M_1} (s') ds'} + \abs{\int_\mathcal{S} T_2 (s' \mid s,a) \bs{V^*_{M_2} (s') - V^*_{M_1} (s')}ds'}} \\
\leq~& \beta_R \cdot R_{\mathrm{max}} + \gamma \max_a \bcc{\norm{T_1 (\cdot \mid s,a) - T_2 (\cdot \mid s,a)}_1 \sup_{s'} V^*_{M_1} (s') + \norm{T_2 (\cdot \mid s,a)}_1 \sup_{s'} \abs{V^*_{M_1} (s') - V^*_{M_2} (s')} } \\
\leq~& \beta_R \cdot R_{\mathrm{max}} + \gamma \max_a \bcc{\norm{T_1 (\cdot \mid s,a) - T_2 (\cdot \mid s,a)}_1 \frac{R_{\mathrm{max}}}{1-\gamma} + 1 \cdot \sup_{s'} \abs{V^*_{M_1} (s') - V^*_{M_2} (s')} } \\
\leq~& \beta_R \cdot R_{\mathrm{max}} + \frac{\gamma \cdot \beta_T \cdot R_{\mathrm{max}}}{1-\gamma} + \gamma \sup_{s'} \abs{V^*_{M_1} (s') - V^*_{M_2} (s')}  
\end{align*}
Taking sup over $s$, and rearranging the terms complete the proof. 
\end{proof}

\begin{lemma}[adapted from \cite{even2003approximate}]
\label{lemma:simulation}
Suppose we have two $\brr{\beta_R, \beta_T}$ approximately-equivalent MDPs $M_1$ and $M_2$. Then for any policy $\pi: \mathcal{S} \rightarrow \Delta\brr{\mathcal{A}}$, we have:
\[
\sup_s \abs{V^{\pi}_{M_1} \brr{s} - V^{\pi}_{M_2} \brr{s}} ~\leq~ \frac{\beta_R \cdot R_{\mathrm{max}}}{1 - \gamma} + \frac{\gamma \cdot \beta_T \cdot R_{\mathrm{max}}}{\brr{1 - \gamma}^2} . 
\]
\end{lemma}

\begin{lemma}
\label{near-opt-lemma}
Consider an MDP $M = \brr{\mathcal{S}, \mathcal{A}, T, R, P_0, \gamma}$ with a unique deterministic optimal policy $\pi^*_M: \mathcal{S} \rightarrow \mathcal{A}$ such that the following holds:
\begin{equation}
\label{eq:use-temp-1}
V^*_M \brr{s} ~=~ Q^*_M \brr{s,\pi^*_M\brr{s}} ~\geq~ Q^*_M \brr{s,a'} + \delta , \quad \forall{a' \neq \pi^*_M\brr{s}, s \in \mathcal{S}} .
\end{equation}
Then any $\epsilon_{\mathrm{opt}}$-near optimal policy $\pi: \mathcal{S} \rightarrow \Delta \brr{\mathcal{A}}$ for $M$ (as given in \eqref{eq:near-opt-policy}) satisfies the following:
\[
\sup_s \norm{\pi^*_{{M}}\brr{\cdot \mid s} - \pi\brr{\cdot \mid s}}_1 ~\leq~ \frac{\epsilon_\mathrm{opt}}{\delta} .
\]
Further, any other deterministic policy $\pi:\mathcal{S} \rightarrow \mathcal{A}$ cannot be $\epsilon_{\mathrm{opt}}$-near optimal when $\epsilon_{\mathrm{opt}} < \delta$.
\end{lemma}

\begin{proof}
Denote $f\brr{s,\pi\brr{s}} = \expctover{a \sim \pi\brr{\cdot \mid s}}{f\brr{s,a}}$. First, we note that for any $s \in \mathcal{S}$, we have
\begin{equation}
\label{eq:use-temp-2}
Q^*_M\brr{s,\pi\brr{s}} ~\geq~ Q^\pi_M\brr{s,\pi\brr{s}} ~=~ V^\pi_M \brr{s} . 
\end{equation}
Fix any state $s \in \mathcal{S}$. Let $a_s^* = \pi^*_M\brr{s}$, $a'_s$ be one of the second best action after $a^*_s$, and $p = \pi\brr{a_s^* \mid s}$. Then, we have
\begin{align}
Q^*_M\brr{s,\pi\brr{s}} ~=~& \expctover{a \sim \pi\brr{\cdot \mid s}}{R\brr{s,a} + \expctover{s' \sim T\brr{\cdot \mid s,a}}{V^*_M\brr{s'}}} \nonumber \\
~\leq~& p \cdot Q^*_M\brr{s,a_s^*} + \overline p \cdot Q^*_M\brr{s,a'_s} , \label{eq:use-temp-3}
\end{align}
where $\bar{p} = 1-p$. Consider
\begin{align*}
\overline p \cdot \delta ~\stackrel{(i)}{\leq}~& \overline p \cdot \bcc{Q^*_M\brr{s,a_s^*} - Q^*_M\brr{s,a'_s}} \\
~\stackrel{(ii)}{\leq}~& Q^*_M\brr{s,a_s^*} - Q^*_M\brr{s,\pi\brr{s}} \\
~\stackrel{(iii)}{\leq}~& V^*_M \brr{s} - V^\pi_M \brr{s} \\
~\leq~& \epsilon_\mathrm{opt} .
\end{align*}
where $(i)$, $(ii)$, and $(iii)$ are due to Eqs.~\eqref{eq:use-temp-1},~\eqref{eq:use-temp-3},~and~\eqref{eq:use-temp-2} respectively. Thus, we get
\[
\norm{\pi^*_{{M}}\brr{\cdot \mid s} - \pi\brr{\cdot \mid s}}_1 ~=~ \overline p ~\leq~ \frac{\epsilon_\mathrm{opt}}{\delta} .
\]
When $\epsilon_{\mathrm{opt}} < \delta$ and $\pi$ is deterministic s.t. $\pi \neq \pi^*_M$, from Eqs.~\eqref{eq:use-temp-1},~and~\eqref{eq:use-temp-2}, we have:
\[
V^*_M \brr{s} ~\geq~ V^\pi_M \brr{s} + \delta ~>~ V^\pi_M \brr{s} + \epsilon_\mathrm{opt} .
\]
\end{proof}

\paragraph{Proof of Theorem~\ref{thm:reward-only-shaping}.}

\begin{proof}
Note that $\widehat{M}$ and $\overline{M}$ are $\brr{0, \beta_T}$ approximately-equivalent. We first show that $\pi_{\widehat{M}}$ is near-optimal in $\overline{M}$, then using Theorem~\ref{thm:main-change-T} completes the proof. For any $s \in \mathcal{S}$, we have:
\begin{align*}
& \abs{V^*_{\overline M} \brr{s} - V^{\pi_{\widehat{M}}}_{\overline M} \brr{s}} \\
~\leq~& \abs{V^*_{\overline M} \brr{s} - V^*_{\widehat M} \brr{s}} + \abs{V^*_{\widehat M} \brr{s} - V^{\pi_{\widehat{M}}}_{\widehat M} \brr{s}} + \abs{V^{\pi_{\widehat{M}}}_{\widehat M} \brr{s} - V^{\pi_{\widehat{M}}}_{\overline M} \brr{s}} \\
~\leq~& \frac{\gamma \cdot \beta_T \cdot R_\phi^{\mathrm{max}}}{\brr{1-\gamma}^2} + \epsilon_\mathrm{opt} + \frac{\gamma \cdot \beta_T \cdot R_\phi^{\mathrm{max}}}{\brr{1-\gamma}^2} , 
\end{align*}
where the last inequality is due to Lemma~\ref{lemma:robust-diff} and Lemma~\ref{lemma:simulation}. 
\end{proof}

\paragraph{Proof of Theorem~\ref{thm:dynamics-only-shaping}.}

\begin{proof}
Note that $\widehat{M}$ and $\overline{M}$ are $\brr{\beta_R , 0}$ approximately-equivalent. We first show that $\pi_{\widehat{M}}$ is near-optimal in $\overline{M}$, then using Theorem~\ref{thm:main-change-T} completes the proof. For any $s \in \mathcal{S}$, we have:
\begin{align*}
& \abs{V^*_{\overline M} \brr{s} - V^{\pi_{\widehat{M}}}_{\overline M} \brr{s}} \\
~\leq~& \abs{V^*_{\overline M} \brr{s} - V^*_{\widehat M} \brr{s}} + \abs{V^*_{\widehat M} \brr{s} - V^{\pi_{\widehat{M}}}_{\widehat M} \brr{s}} + \abs{V^{\pi_{\widehat{M}}}_{\widehat M} \brr{s} - V^{\pi_{\widehat{M}}}_{\overline M} \brr{s}} \\
~\leq~& \frac{\beta_R \cdot \max \bcc{R_\mathrm{max}, R_\phi^{\mathrm{max}}}}{{1-\gamma}} + \epsilon_\mathrm{opt} + \frac{\beta_R \cdot \max \bcc{R_\mathrm{max}, R_\phi^{\mathrm{max}}}}{{1-\gamma}} , 
\end{align*}
where the last inequality is due to Lemma~\ref{lemma:robust-diff} and Lemma~\ref{lemma:simulation}. 
\end{proof}

\section{Additional Experiments}\label{sec.appendix.experiments}

\begin{figure*}[t!]
	\begin{subfigure}{0.45\linewidth}
		\includegraphics[width=1.15\linewidth]{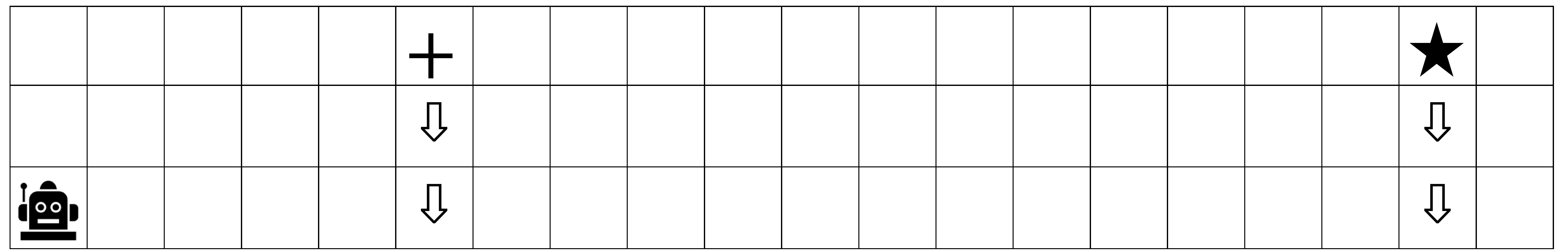}
    	\caption{Environment}
        \label{fig:2d.env}
	\end{subfigure}
	\quad \quad \quad
	\begin{subfigure}{0.45\linewidth}
		\includegraphics[height=.8\linewidth]{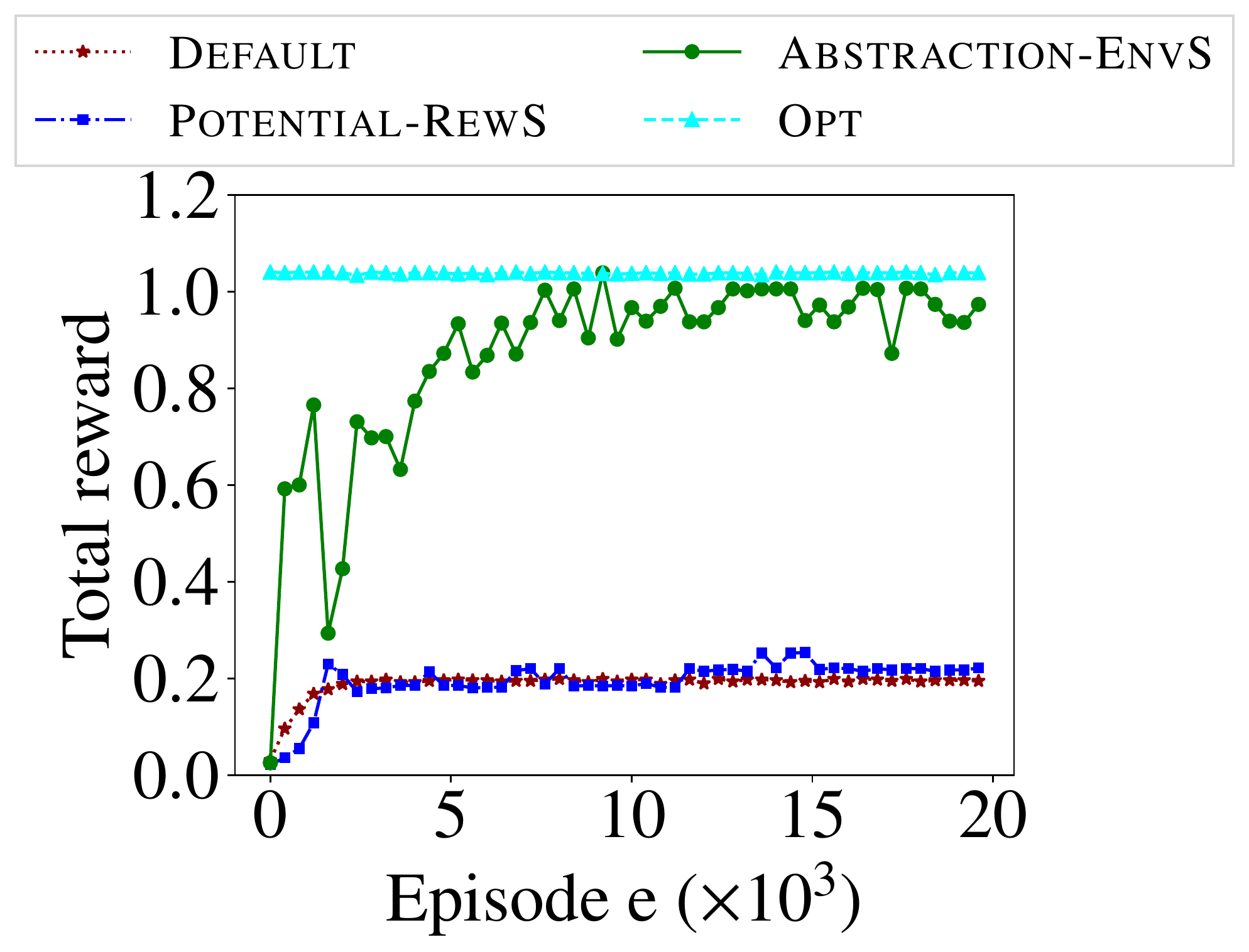}
		\vspace{-4mm}
    	\caption{Performance}
        \label{fig:results.2}
	\end{subfigure}
	\caption{(a) An example of the \textit{catcher} environment with the actions $\mathcal{A}$=\{"left", "right"\} and the state space size $|\mathcal{S}|=20\times60^2$. (b) Learning curves for \textit{catcher} task, averaged across 30 runs.}
\end{figure*}

In this section, we performed an additional experiment to showcase the efficacy of our algorithm.

% In this section we performed additional experiments to showcase the power of our algorithm for the \textit{catcher} game.  

\subsection{Environment Details}
We consider the \textit{catcher} environment illustrated in Figure~\ref{fig:2d.env}. In this game, an agent moves along the x-axis, and two types of objects fall from above the agent, perpendicular to the agent's axis of movement. In this game, the agent has two actions, i.e., $\mathcal{A}$=\{``left", ``right"\}. The action taken succeeds in moving the agent in the corresponding direction with a probability of $0.99$. And with a probability of $0.01$, the movement happens in the opposite direction. The agent starts from the left corner of the environment, and its goal is to move itself to catch the falling objects. The object ``plus" always appears close to the agent's initial position, e.g., on the $5^\mathrm{th}$ column, and the object ``star" appears before the last column. Intersecting with the object ``star" gives the reward of $1.0$, and intersecting with the object ``plus" gives the reward of $0.04$. After catching or missing the objects, they appear in the starting position. The game's episode ends when the maximum time step $H=30$ is reached, or the agent has caught the ``star" object. Since we have the number of the cells to be $60$ and our state representation is described as the concatenation of the binary vector of the position of the agent and the locations of each object, it results in $20\times60^2$ total possible states, i.e., $|\mathcal{S}|=20\times60^2$. The rewards are discounted at $\gamma=0.99$.  The optimal policy $\pi^*$ is to catch the ``star"  object.

% In this game agent moves along the $x$ axis and two types of objects fall from above the agent, perpendicular to the agent’s axis of movement. In this game agent has two action, i.e., $\mathcal{A}$=\{"left", "right"\} which moves the agent in the corresponding direction with success probability $0.99$ and with probability $0.01$ the move opposite direction is performed. 

% The agent starts on the left corner of the environment and it’s goal is to move itself to catch the falling objects. Object ``plus" always appears close to the agent's initial position, e.g., on the $5th$ column and an object ``star" appears before the last column, see Figure \ref{fig:2d.env}. Intersecting with the object ``star" gives the reward of $1.0$ and intersecting with the object ``plus" gives the reward of $0.04$. After catching or missing the objects they appear in the starting position. 

% The game's episode ends when the maximum time step $H=30$ is reached, or the agent has caught the ``star" object. Since we have the number of the cells to be $60$ and our state representation is described as the concatenation of the binary vector of the position of the agent and the position of each object, it will give us rise to $20\times60^2$ total possible states, i.e., $|\mathcal{S}|=20\times60^2$. The rewards are discounted at $\gamma=0.99$.  The optimal policy $\pi^*$ is to catch the ``star"  object.

\subsection{Experimental Setup}

We use the vanilla policy gradient method as the agent's algorithm, where the policy network is implemented as a two-layer neural network with \texttt{relu} activation function and \texttt{softmax} on the last layer. A one-hot vector of length $20$, corresponding to the number of cells in the x-axis, represents the agent's location. One-hot vectors of length $60$, corresponding to the number of cells, represent the positions of all two objects in the environment. Hence the length of the input vector to the neural network is $20+2\times60=140$.

% We use the vanilla policy gradient method as the agent's algorithm, where the policy network is implemented as a two-layer neural network with \texttt{relu} activation function and \texttt{softmax} on the last layer. A one-hot vector of length $20$, corresponding to the $x$ axis of the environment represents the agent's location and one-hot vector of length $60$ corresponding to the number of cells the positions of all two objects in the environment. Hence the length of the input vector to the neural network is $20+2\times60=140$.

\paragraph{Methods evaluated.} 

As in Section~\ref{sec:experiments}, the following algorithms were considered in the comparison: (i) default baseline with training on original $M$ (\textsc{Default}), (ii) baseline of potential-based reward shaping of $M$  (\textsc{Potential-RewS}) (see \cite{ng1999policy}), (iii) training on $\overline{M}$ shaped using our approach (\textsc{Abstraction-EnvS}), and (iv) optimal policy that catches object ``star" (\textsc{Opt}). 

For the \textsc{Potential-RewS} approach, we need to approximate the optimal value function  $V_{M}^{*}$ of the original MDP $M$. To this end, we first run the policy gradient algorithm up to $50,000$ iterations, and then used resulting policy $\pi$ to approximate $V_{M}^{*}$ by $V_M^{\pi}$. To estimate the value function of the policy $\pi$, we used Monte Carlo (MC) prediction algorithm with $200,000$ samples. Given that we have domain knowledge of our problem, e.g., the information about the agent's position and ``star" object's position is sufficient to learn the optimal behavior, our abstraction $\phi:\mathcal{S}\rightarrow \mathcal{X}_\phi$ can be defined as follows: we consider only the position of the agent and position of the ``star" object to encode our state space. This abstraction reduces the state-space from $20\times60^2$ down to $20\times60$, i.e.,  $|\mathcal{X}_\phi| = 20\times60$. We construct the MDP $M_\phi$ by first obtaining ${\widetilde{M}}_{\phi,\brr{2,1}}$ and then using the rewards only attack as in \cite{rakhsha2020policy}. 

We evaluate the performance of the four algorithms mentioned above and report the results by averaging across $30$ runs and with random initialization for each run (\emph{cf.}, Figure~\ref{fig:results.2}). 

}

%%%%%%%%%%%%%%%%%%%%%%%%%%%%%%%%%%%%%%%%%%%%%%%%%%%%%%%%% END
\end{document}